\documentclass{article}
\usepackage{iclr2026_conference,times}

\usepackage{amsmath,amsfonts,bm}

\def\eqref#1{equation~\ref{#1}}

\def\1{\bm{1}}

\DeclareMathAlphabet{\mathsfit}{\encodingdefault}{\sfdefault}{m}{sl}
\SetMathAlphabet{\mathsfit}{bold}{\encodingdefault}{\sfdefault}{bx}{n}

\newcommand{\E}{\mathbb{E}}

\usepackage{hyperref}
\usepackage{url}
\usepackage{booktabs}
\usepackage{array}
\usepackage{amsmath, amssymb, empheq}
\usepackage[nameinlink,noabbrev]{cleveref}
\usepackage{microtype}
\allowdisplaybreaks
\usepackage{xcolor}
\usepackage{subcaption}
\usepackage{graphicx}
\usepackage{makecell}
\usepackage[most]{tcolorbox}
\usepackage{wrapfig}

\definecolor{newblue}{HTML}{0064E0}

\definecolor{nblue}{HTML}{E3EDFF}

\definecolor{ngreen}{HTML}{E7FCF2}

\definecolor{nred}{HTML}{FFE8ED}

\title{ExecTune: Effective Steering of Black-Box LLMs with Guide Models}

\iclrfinalcopy

\author{Vijay Lingam\thanks{Corresponding author.}\\
AWS AI\\
\href{mailto:vjlingam@amazon.com}{\texttt{vjlingam@amazon.com}}
\And
Aditya Golatkar\\
AWS AI\\
\And
Anwesan Pal\\
AWS AI
\And
Ben Vo\\
AWS AI
\And
Narayanan Sadagopan\\
AWS AI
\And
Alessandro Achille\\
AWS AI
\And
Jun Huan\\
AWS AI
\And
Anoop Deoras\\
AWS AI
\And
Stefano Soatto\\
AWS AI
}

\newcommand{\GCoP}{\textsc{GCoP}}

\usepackage{amsthm}

\newtheorem{theorem}{Theorem}

\DeclareMathOperator{\TV}{TV}

\newtcolorbox[auto counter]{mainbox}[2][]{
  colframe=ngreen,
  colback=ngreen,
  breakable,
  before upper={{\textbf{Main Contributions.} #2}},
  #1
}

\begin{document}

\maketitle

\begin{abstract}

For large language models deployed through black-box APIs, recurring inference costs often dominate one-time training costs, motivating composed agentic systems that amortize expensive reasoning into reusable intermediate representations. We study a broad class of such systems, termed \emph{Guide--Core Policies} (\GCoP), in which a \emph{guide} model generates a structured \textit{strategy} that is executed by a black-box \emph{core} model. This abstraction subsumes base, supervised, and advisor-style approaches, which differ primarily in how the guide is trained. We formalize \GCoP~under a cost-sensitive utility objective and show that end-to-end performance is governed by \emph{guide-averaged executability}: the probability that a strategy can be faithfully followed by the core. Our analysis reveals that existing instantiations of \GCoP~often fail to optimize executability under deployment constraints, leading to brittle strategies and inefficient computation. Guided by these insights, we propose \textsc{ExecTune}, a principled training recipe that combines teacher-guided acceptance sampling, supervised fine-tuning, and structure-aware reinforcement learning to directly optimize syntactic validity, execution success, and cost efficiency. Across mathematical reasoning and code-generation benchmarks, \GCoP~with \textsc{ExecTune} improves accuracy by up to \textbf{9.2\%} over prior state-of-the-art baselines while reducing inference cost by up to \textbf{22.4\%}. \GCoP~with \textsc{ExecTune} enables Claude Haiku-3.5 to surpass Sonnet-3.5 on math and code tasks and comes within \textbf{1.7\%} absolute accuracy of Sonnet 4 at \textbf{38\%} lower cost. Beyond efficiency, GCoP enables modular adaptation by updating guides without retraining the core.\looseness=-1

\end{abstract}

\section{Introduction}
\label{sec:introduction}

Large language models (LLMs) have demonstrated strong performance on complex reasoning and programming tasks, yet their deployment at scale remains constrained by inference cost, latency, and limited adaptability under black-box access. As inference is invoked repeatedly in downstream applications, these recurring costs often dominate one-time training expenses, motivating the design of agentic systems that amortize expensive reasoning into reusable intermediate representations or memories. Thus, agentic inference can be formally written as a cost-sensitive net utility objective~\citep{achille2025ai, zabounidis2025re, kleinman2025e1}
\begin{equation}
    \boxed{J^\pi(s_0) = V^\pi(s_0) - \lambda T^\pi(s_0)}
\end{equation}
where, $\pi$ is the agentic policy, $s_0$ an initial state (or task), and $V^\pi$ denotes the expected task reward, $T^\pi$ denotes the expected inference-time cost, and $\lambda$ controls the trade-off between performance and computation. This formulation captures realistic deployment settings in which marginal improvements in accuracy must be weighed against increased latency and expense.

This objective suggests a simple design principle: separate expensive deliberation from cheap execution. To address this challenge, we propose a broad class of composed policies, \emph{Guide-Core Policies} (GCoP) that decomposes reasoning into two stages: a \emph{small trainable-guide} model generates a high-level plan, advice, or strategy, and a \emph{small black-box core} model executes it to produce the final output for a given task. \GCoP~improves efficiency and reliability by making abstract reasoning learnable and reusable across tasks while keeping execution \emph{cheap and robust}. Moreover, \GCoP~subsumes and outperforms baselines like prompting, supervised fine-tuning, and~\textsc{Advisor} models~\citep{asawa2025train,li2025matryoshka}, using a composed policy whose behavior is governed jointly by guide generation, execution reliability, and computational cost.

Our analysis identifies executability as the key component in GCoP: the probability that a guide-produced strategy is parseable and can be successfully followed by the core. We show via a student–teacher mixture analysis (see \autoref{fig:sft-passk}) that the performance gap to a large black-box baseline is controlled by guide-averaged executability; we reduce this gap through \emph{guide-training} which explicitly raises executability. The result is structured, repeatable, reliable plans—strategies that the core can consistently execute—yielding higher utility by turning high-level reasoning into predictable, reusable action (\autoref{app:student_teacher_mixture}).

Despite this promise, existing guide–core and advisor-style instantiations rarely train the guide for the quantity that matters in GCoP: executability. Prompted or supervised guides can be informative yet underspecified, while advisor-style models typically optimize advice quality in isolation rather than the downstream constraints imposed by a smaller, cost-constrained core. Consequently, they often generate strategies that are difficult to parse, hard to follow, or mismatched to the core’s capabilities, leading to brittle behavior and wasted inference (see \autoref{app:qualitative_analysis}).

Guided by our theoretical insights, we propose \textsc{ExecTune}, a principled training recipe for learning effective executable guides. \textsc{ExecTune} proceeds in two stages. First, it initializes the guide using teacher-guided acceptance sampling and supervised fine-tuning to bias generation toward executable strategies. Second, it performs structure-aware reinforcement learning with explicit rewards for syntactic validity, execution success, and cost-efficient behavior, while penalizing brittle or unparseable outputs. This design directly aligns guide training with the utility objective induced by GCoP. \looseness=-1

\begin{mainbox}

\begin{enumerate}
    \item \textbf{\GCoP~framework.} We formalize \textbf{\GCoP}\ as a composed guide-core policy abstraction, where a trainable guide model generates context that a black-box core executes. Our theory building on net-utility objective identifies \emph{guide-averaged executability} as the key factor controlling the value gap to stronger core-only baselines.

    \item \textbf{Teacher-guided acceptance sampling for strategy SFT.}
    We construct high-executability strategy data by using a strong teacher to propose strategies and the target core to validate them, yielding an SFT initialization tailored to the core's execution constraints (\autoref{sec:accepted_target_main}).

    \item \textbf{\textsc{ExecTune}: executable strategy post-training.}
    We propose \textsc{ExecTune}, which refines the guide with structure-aware GRPO using rewards for (i) correct strategy format/parseability and (ii) execution success, and penalties for missing, malformed, or digressive strategies (\autoref{sec:exectune_main}).

    \item \textbf{Improved utility at lower costs.} Across math and code benchmarks, \GCoP\ with \textsc{ExecTune} improves accuracy by up to \textbf{9.2\%} while reducing inference cost by up to \textbf{22.4\%} over prior black-box steering baselines, and approaches next-generation cores (within \textbf{1.7\%} of Sonnet-4) at \textbf{38\%} lower cost (See~\autoref{fig:reward_cost_tradeoff}).
\end{enumerate}
\end{mainbox}

Beyond efficiency, \GCoP~enables modular adaptation: guides can be updated for domain adaptation, continual learning, or targeted unlearning without retraining the core. This modularity is particularly valuable in black-box deployment settings, where the core model cannot be modified directly.

\begin{figure}[h]
    \centering
    \includegraphics[width=1\columnwidth]{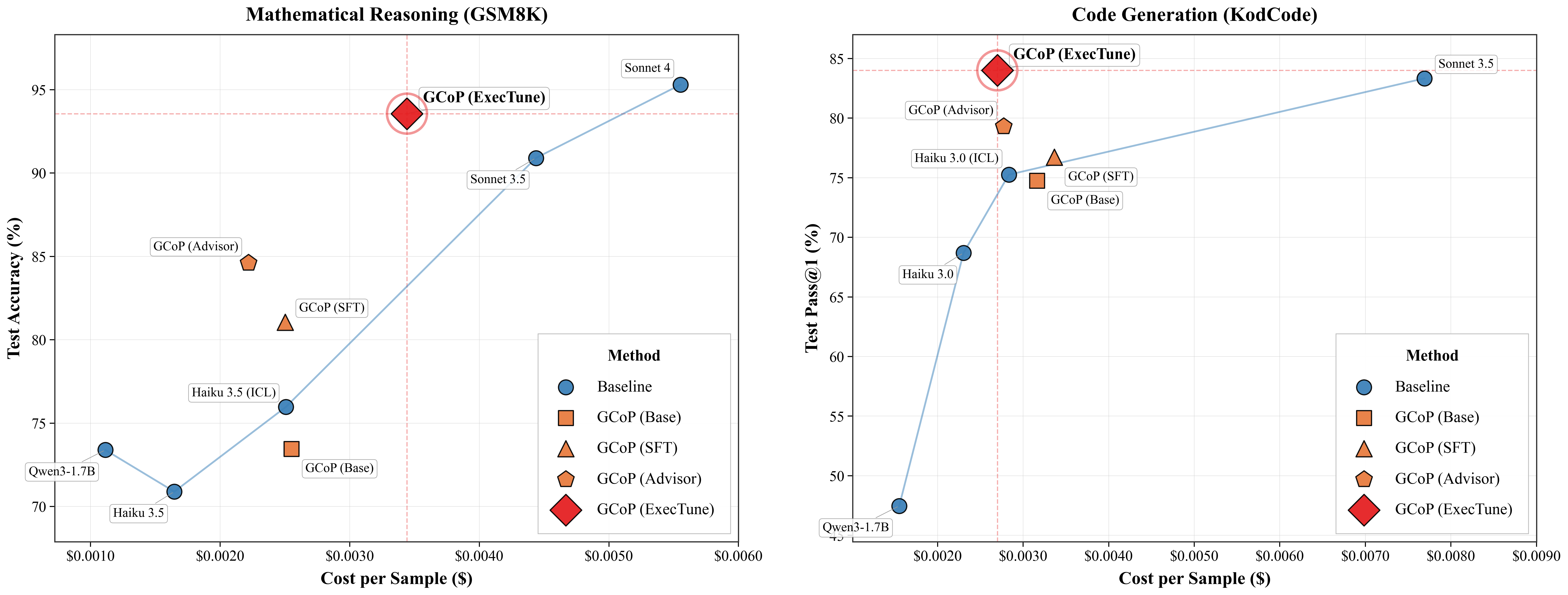}
    \caption{\textbf{Reward--cost trade-off.} Test performance versus \emph{total inference cost} on GSM8K (left; core=\textbf{Haiku-3.5}) and KodCode (right; core=\textbf{Haiku-3.0}). A small \textbf{ExecTune} guide (\textbf{Qwen3-1.7B}) yields the best accuracy--cost trade-off among \GCoP{} variants, outperforming prompting/ICL and \GCoP(Base/SFT/Advisor) while remaining far cheaper than frontier baselines. \GCoP(ExecTune) matches or exceeds \textbf{Sonnet-3.5} and approaches \textbf{Sonnet-4} at substantially lower cost, showing that executable strategies can reliably steer black-box cores toward large-model performance.}
    \label{fig:reward_cost_tradeoff}
\end{figure}

\section{Related Work}
\label{sec:related_work}
Recent advances in large language models (LLMs) have sparked growing interest in methods that adapt black-box APIs to user needs without requiring model retraining or internal access. Our work intersects with three key research threads: (1) prompt optimization for black-box LLMs, (2) learned steering via guide models and controllers, and (3) knowledge distillation and cost-efficient deployment strategies. This section surveys each area and situates our framework in the broader landscape.

\subsection{Prompt Optimization for Black-Box LLMs}
Prompt optimization methods aim to elicit desirable behaviors from LLMs through natural language instructions, without modifying model parameters. Early strategies included gradient-free techniques like AutoPrompt~\citep{shin2020autoprompt}, which searched for discrete token triggers. Black-box-compatible methods have since emerged, including evolutionary and bandit-style approaches such as PromptWizard~\citep{agarwal2024promptwizard}, PromptBreeder~\citep{fernando2023promptbreeder}, and BBPO~\citep{cheng2024black}, which optimize prompts through trial-and-error. GEPA~\citep{agrawal2025gepa}, \textsc{DoT}~\citep{dot} and Self-Refine~\citep{madaan2023self} introduce reflective or self-critiquing mechanisms, showing that language models themselves can evaluate and improve prompts via feedback loops.

These methods generate task-general or input-agnostic prompts. However, static optimization may fall short for dynamic tasks or personalization. This limitation motivates adaptive prompting techniques that produce input-conditional guidance at inference time.

\subsection{Learned Steering with Guide Models}
Guide models generate instance-specific natural language instructions to steer black-box LLMs. Guide Models~\citep{asawa2025train} and Directional Stimulus Prompting~\citep{li2023guiding} train lightweight controllers to craft tailored prompts, using a combination of supervised learning and reinforcement learning (RL) from the LLM's feedback. Matryoshka Pilot (M-Pilot)~\citep{li2025matryoshka} treats the LLM as an environment, with a learned guide breaking complex tasks into subtasks across multiple turns.

These approaches echo structured prompting techniques like Least-to-Most Prompting~\citep{zhou2022least} and Program-of-Thought prompting~\citep{chen2022program}, which inject intermediate reasoning steps to improve task decomposition. At decoding time, methods like FUDGE~\citep{yang-klein-2021-fudge} and DExperts~\citep{liu-etal-2021-dexperts} steer generation through learned discriminators or ensemble scoring, though they require token-level access. Guide models operate fully at the language level and work even with closed-source APIs. They complement alignment frameworks such as InstructGPT~\citep{ouyang2022training} and Constitutional AI~\citep{bai2022constitutional}, which apply RLHF or AI-based critiques. Prompt rewriting frameworks like BPO~\citep{cheng2024black} further demonstrate that modifying the user's query alone can boost output quality from API-tuned models. Tool-augmented prompting methods such as ReAct~\citep{yao2022react} and Auto-CoT~\citep{zhang2022automatic} highlight how LLMs can dynamically generate intermediate reasoning or query tools to improve response quality.

Our work builds on these principles, using a two-stage approach: an offline supervised step followed by online RL, improving robustness under limited feedback. Unlike prior work, we focus on leveraging learned guides for scale-bridging and cost-sensitive inference.

\subsection{Knowledge Distillation and Cost Efficiency}
To make LLM deployment more affordable, knowledge distillation techniques aim to transfer capabilities from large models to smaller ones. In black-box contexts, this is done using only the teacher's outputs. Step-by-step distillation~\citep{hsieh2023distilling} and Chain-of-Thought distillation~\citep{do2025effectiveness} show that intermediate rationales help teach reasoning to smaller models. Recent approaches like MiniLLM~\citep{gu2023minillm} and Direct Preference Knowledge Distillation (DPKD)~\citep{li2024direct} use reward modeling or reverse-KL objectives to preserve quality and stylistic traits. However, distillation may strip away safety mechanisms: Jahan and Sun~\citep{jahan2025black} showed that models cloned from API outputs can retain task ability while losing alignment safeguards. To reduce inference cost without retraining, system-level methods route queries through model cascades. FrugalGPT~\citep{chen2023frugalgpt} and ABC~\citep{kolawole2024agreement} dynamically assign tasks to models of varying size based on confidence or agreement. Tool-augmented prompting complements these strategies. Toolformer~\citep{schick2023toolformer} and ReAct~\citep{yao2022react} teach LLMs to call APIs mid-generation, enhancing factuality and reasoning.

Our method aligns with these directions by training a guide that enables a compact LLM to benefit from the reasoning capacity of a larger one without direct distillation. The result is a cost-effective, steerable deployment strategy that preserves model modularity and API compatibility.

\section{Method}
\label{sec:method}

\paragraph{Overview.}
We study \emph{Guide-Core Policies} (\GCoP): composed agentic systems where a \emph{trainable guide} (e.g., Qwen3-1.7B~\citep{qwen3}) produces an explicit, \emph{parseable} strategy and a (typically cheaper) \emph{black-box core} (e.g., Claude Haiku 3\footnote{https://www.anthropic.com/news/claude-3-haiku}) executes conditioned on it. Our goal is to understand when such composed systems can match strong but expensive large-model baselines under realistic deployment constraints, and how to train guides whose strategies are reliably executable by the target core model.
This section (i) formalizes \GCoP\ under a cost-sensitive net-utility objective, (ii) identifies \emph{guide-averaged executability} as the primary bottleneck governing the value gap to large-model baselines, and (iii) introduces \textsc{ExecTune}, a training recipe that explicitly aligns guide outputs with downstream execution constraints.

\subsection{Objective: reward--cost net utility}
\label{sec:method_objective}
We model tasks as finite-horizon interactions with reward bounded by $R_{\max}$, and measure inference-time deployment cost (tokens, latency, or monetary) via $T_\pi(\cdot)$.
For any policy $\pi$ and initial state $s_0$, we optimize the net utility
\begin{equation}
\label{eq:reward_cost_objective_main}
\boxed{
J_\pi(s_0) \;\triangleq\; V_H^\pi(s_0) \;-\; \lambda\, T_\pi(s_0),
}
\end{equation}
where $V_H^\pi(s_0) \triangleq \E[\sum_{t=0}^{H-1} r_t \mid s_0,\pi]$ and $\lambda \ge 0$ trades off performance and compute.
This objective captures deployment regimes where small gains in accuracy must be weighed against recurring inference costs.

\subsection{\GCoP~as a composed policy family}
\label{sec:gcop_family_main}
Let $\mathcal{Z}$ denote a \emph{strategy} space (e.g., plan, advice, program sketch, structured prompt).
A guide policy $\pi_g(z\mid s)$ samples a strategy $z\in\mathcal{Z}$ given state $s$, and a black-box core $\pi_c(a\mid s,z)$ produces the final action/output conditioned on $(s,z)$.
Composing guide and core yields the induced policy
\begin{equation}
\label{eq:composite_policy_main}
\boxed{
\pi_{gc}(a\mid s)\;\triangleq\;\sum_{z\in\mathcal{Z}} \pi_g(z\mid s)\,\pi_c(a\mid s,z).
}
\end{equation}
Different \GCoP\ instantiations correspond to different ways of obtaining $\pi_g$ (Base, SFT, advisor-style, or ours), while keeping the target black-box core model fixed.

\subsection{Executability controls the value gap}
\label{sec:executability_main}
A central failure mode in guide--executor systems is \emph{non-executability}: the guide outputs strategies that are not parseable or not faithfully followable by the target core, leading to wasted compute and brittle behavior.
We formalize this via a per-(state,strategy) success probability $q(s,z)\triangleq \Pr(G=1\mid s,z)$, where $G=1$ denotes ``good execution'' under the environment/validator.

We define the \emph{guide-averaged executability} as:
\begin{equation}
\label{eq:alpha_def_main}
\boxed{
\alpha(s)\;\triangleq\;\E_{z\sim \pi_g(\cdot\mid s)}[q(s,z)].
}
\end{equation}
Under a student--teacher mixture view (detailed in Appendix~\ref{app:student_teacher_mixture}), the induced policy $\pi_{gc}$ can be expressed as a mixture between a strong teacher policy $\pi_L$ and an aggregate ``bad-execution'' component, weighted by $\alpha(s)$.
This yields the following bound on the value gap:
\begin{theorem}[Value gap controlled by executability]
\label{thm:value_gap_alpha_main}
Under the mixture model assumptions in Appendix~\ref{app:student_teacher_mixture}, for any $s_0$,
\begin{equation}
\label{eq:value_gap_thm_main}
\boxed{
V_H^{\pi_L}(s_0)-V_H^{\pi_{gc}}(s_0)\;\le\;2HR_{\max}\,\E_{s\sim d_L}\!\big[1-\alpha(s)\big],
}
\end{equation}
where $d_L$ is the (average) state visitation distribution under $\pi_L$.
\end{theorem}
Thus, to close the gap to strong cores, the guide must \emph{increase $\alpha(s)$ on task-relevant states}, i.e., produce strategies that the target core can reliably execute.

\subsection{Accepted-strategy target distribution for strategy SFT}
\label{sec:accepted_target_main}
To increase $\alpha(s)$, we build a strategy dataset that matches the target core's execution constraints.
We use \emph{teacher-guided acceptance sampling}: a strong teacher proposes strategies $z$, and we \emph{accept} those for which the target core (or validator) achieves high empirical success when executing $z$.
This defines an \emph{accepted-strategy} distribution $\pi_{\mathrm{acc}}(\cdot\mid s)$ (see Appendix~\ref{app:accepted_sampling}).

Acceptance sampling provably reweights toward high-executability strategies:
\begin{theorem}[Accepted strategies have high expected success]
\label{thm:accepted_high_q_main}
Fix $s$ and accept $z$ iff $\hat q_K(s,z)\ge \tau$ using $K$ validation trials. Then for any $\eta>0$,
\begin{equation}
\label{eq:accepted_q_lower_bound_main}
\boxed{
\E_{z\sim \pi_{\mathrm{acc}}(\cdot\mid s)}[q(s,z)]
\;\ge\;
(\tau-\eta)\Big(1-\frac{e^{-2K\eta^2}}{A_s}\Big),
}
\end{equation}
where $A_s$ is the acceptance rate.
\end{theorem}
We then initialize the guide by SFT on $(s,z)\sim \pi_{\mathrm{acc}}(\cdot\mid s)$, which directly trains the guide toward strategies the core can execute.

\begin{figure}[t]
    \centering
    \includegraphics[width=0.95\columnwidth]{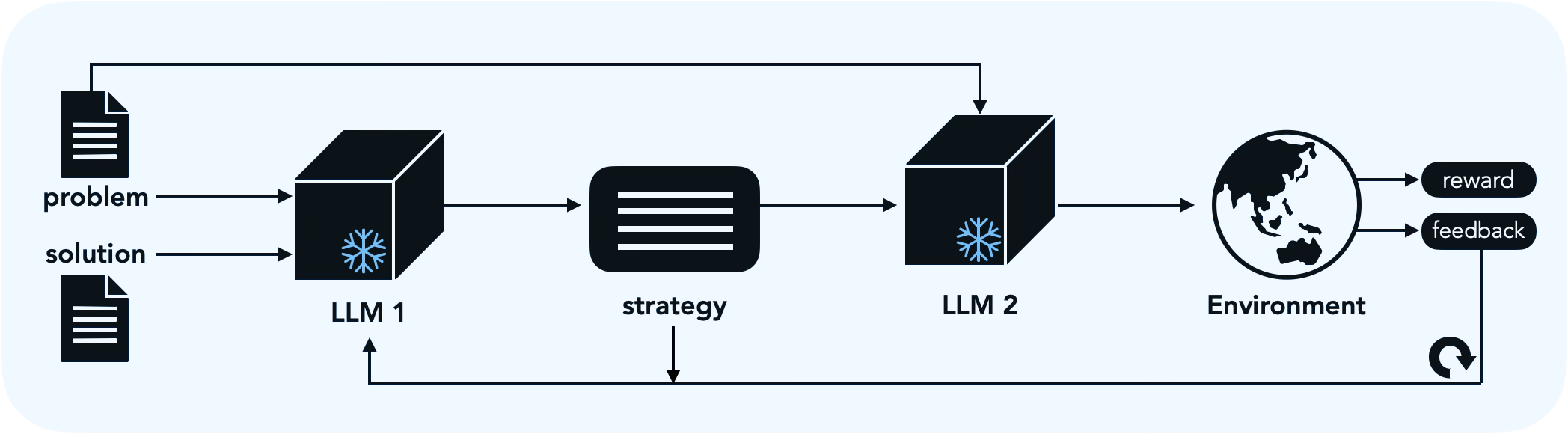}
    \caption{SFT dataset curation pipeline. A strong model (LLM-1; e.g., Claude Sonnet 4.5) extracts a high-level strategy from a (problem, solution) pair. A weaker frontier model (LLM-2; e.g., Claude Haiku 3) generates a solution conditioned on the problem and strategy. An environment provides reward and feedback; unsuccessful strategies are iteratively refined by LLM-1. The resulting (problem, strategy) pairs form the SFT corpus.}
    \label{fig:sft_recipe}
\end{figure}

\subsection{\textsc{ExecTune}: executable guide post-training}
\label{sec:exectune_main}

We propose \textsc{ExecTune}, a two-stage recipe that aligns guide generation with downstream executability and net utility.

\paragraph{Stage 1: teacher-guided acceptance sampling + SFT.}
We curate executable strategies using the pipeline in Fig.~\ref{fig:sft_recipe} and fine-tune the guide to imitate accepted strategies, improving $\alpha(s)$ by construction.

\paragraph{Stage 2: structure-aware GRPO with executability shaping.}
We further refine the guide with GRPO~\citep{grpo} algorithm using a shaped reward that enforces a structured guide--core interface, penalizes leakage of final answers into the strategy, and discourages strategies that harm core behavior.

\textbf{Structure validity.} We require the guide to emit a well-formed \texttt{<strategy>}...\texttt{</strategy>} block that is parseable by a deterministic checker. Let
\[
\mathbb{I}_{\mathrm{str}} \triangleq \mathbf{1}\{\texttt{<strategy>} \text{ block is present and well-formed}\}.
\]

\textbf{Judge-based strategy quality and non-leakage.}
Let $J(s,z)\in[0,1]$ be an LLM-as-a-Judge score applied to the \emph{parsed} strategy, capturing strategy quality and verifying that it does not contain the final answer (e.g., full code). We treat $J$ as an additional shaping signal.

\textbf{No-negative-behavior shaping.}
Let $y$ denote the final output produced by the core. We define a \emph{non-degradation} term that penalizes strategies that reduce performance relative to running the core without the guide:
\begin{equation}
\label{eq:no_negative_behavior_main}
\Delta R(s,z)
\;\triangleq\;
R\!\big(\pi_c(\cdot \mid s,z)\big)\;-\;R\!\big(\pi_c(\cdot \mid s)\big),
\end{equation}
where $R(\cdot)$ denotes the episodic task reward induced by the corresponding core execution distribution.\footnote{In deterministic settings, $R(\pi_c(\cdot\mid \cdot))$ is simply the reward of the realized output $y$.}

\textbf{Final shaped reward.}
Combining these components, we train the guide using the shaped reward
\begin{equation}
\label{eq:exectune_reward_main}
\boxed{
\tilde R
\;\triangleq\;
\mathbb{I}_{\mathrm{str}}
\Big(
R \;+\; \beta
\;+\; \gamma\, J(s,z)
\;-\; \kappa\,\big[-\Delta R(s,z)\big]_+
\Big)
}
\end{equation}
where $\beta\ge 0$ is a small structural bonus, $\gamma\ge 0$ weights the judge score, $\kappa\ge 0$ penalizes degradations, and $[x]_+\triangleq \max\{x,0\}$.
Thus malformed strategies receive $\tilde R=0$ even if the final answer is correct; valid strategies are rewarded for execution success, high-quality non-leaky guidance, and are discouraged from harming baseline core behavior.
Full judge design and additional implementation details are deferred to Appendix~\ref{app:exectune_details}.

\paragraph{Instantiations.}
Our framework subsumes common guide training choices:
(i) \textbf{\GCoP(Base)} (prompted guide; no fine-tuning), (ii) \textbf{\GCoP(SFT)} (guide fine-tuned on curated strategies dataset), (iii) \textbf{\GCoP(\textsc{Advisor})} (guide trained following the recipe in~\citep{asawa2025train}), and (iv) \textbf{\GCoP(\textsc{ExecTune})} (ours), which explicitly optimizes executability under the deployment utility objective in Eq.~\eqref{eq:reward_cost_objective_main}.

\section{Experiments}
\label{sec:experiments}

In this section, we describe the models, datasets, and evaluation protocol used to instantiate \GCoP\ in black-box settings, and present results across mathematical reasoning and code generation.
Implementation and prompt details are deferred to~\cref{app:implementation,app:prompt}.

\paragraph{Training datasets.}
For code generation, we use KodCode~\citep{kodcode}. Following~\citet{murphy}, we randomly sample 1{,}000 problems from KodCode and split them into train/test with an 80:20 ratio. For mathematical reasoning, we train on the GSM8K~\citep{gsm8k} training split.\looseness=-1

\paragraph{Evaluation datasets.}
Guides trained on KodCode are evaluated on the KodCode test split and HumanEval~\citep{humaneval} (out-of-domain code). For mathematical reasoning, we evaluate on the GSM8K test split.

\paragraph{Models and nomenclature.}
We evaluate \GCoP\ with \emph{black-box cores} accessed via proprietary APIs and \emph{open-weight guides}.
We use \textbf{Claude 3.5 Haiku} as the core for GSM8K and \textbf{Claude 3 Haiku\footnote{https://www.anthropic.com/news/claude-3-haiku}} as the core for code tasks (KodCode/HumanEval), chosen because they are cost-efficient yet leave room for improvement.
We report \textbf{Claude 3.5 Sonnet\footnote{https://www.anthropic.com/news/claude-3-5-sonnet}} as a stronger core-only reference point. For guides, we experiment with \textbf{Qwen3-1.7B}.

\paragraph{Baselines and variants.}
We compare against:
(i)~\textbf{ICL} (core-only), where we retrieve the three nearest (problem, solution) training pairs via cosine similarity over prompt embeddings obtained using Cohere Embed v4 model\footnote{https://cohere.com/blog/embed-4} and prepend them to the core prompt;
(ii)~\textbf{Advisor-models}~\citep{asawa2025train}, denoted \textsc{Advisor} in tables, which trains the guide with RL to provide helpful advice to a fixed black-box core.

In \textbf{guided generation} mode, the guide first emits an explicit, parseable strategy $z$, and the core produces the final answer conditioned on both the original problem and $z$.
We evaluate the following guide training recipes (FT Recipe in tables):
\begin{itemize}
    \item \textbf{None}: zero-shot prompted guide (no fine-tuning).
    \item \textbf{SFT}: supervised fine-tuning on accepted strategy data (\autoref{sec:accepted_target_main}).
    \item \textbf{\textsc{Advisor}}: advisor-style RL~\citep{asawa2025train}.
    \item \textbf{\textsc{ExecTune} (ours)}: SFT initialization followed by structure-aware GRPO (\autoref{sec:exectune_main}).
\end{itemize}

\paragraph{Metrics.}
We report Pass@1 for code generation and exact-match accuracy for GSM8K.
Each results table is split into two blocks: (i)~\textit{core-only baselines} (no guide), and (ii)~\textit{guided generation} with the specified guide recipe.

\subsection{Mathematical Reasoning (GSM8K)}
\label{subsec:math_reasoning}

\begin{figure*}[h]
\vspace{-2mm}
\begin{minipage}[t]{0.55\textwidth}
We instantiate \GCoP\ on GSM8K using \textbf{Claude 3.5 Haiku} as the black-box core and \textbf{Qwen3-1.7B} as the guide.
For context, we additionally report the standalone performance of Qwen3-1.7B and Claude 3.5 Sonnet under standard prompting.

\paragraph{SFT strategy data construction.}
We construct the SFT corpus using teacher-guided acceptance sampling (Fig.~\ref{fig:sft_recipe}).
A strong teacher (by default \textbf{Claude Sonnet 4.5}) proposes an initial strategy.
We concatenate the strategy with the original problem and query the \emph{target core} (\textbf{Claude 3.5 Haiku}) to generate a solution, scored by exact-match on GSM8K.
If the strategy fails, the teacher refines it using feedback and we repeat for a small number of iterations; accepted $(\text{problem},\text{strategy})$ pairs form the SFT corpus.
We also ablate different teachers (e.g., Sonnet 4, Sonnet 3.5) while keeping the target core fixed; Fig.~\ref{fig:sft-passk} summarizes the effect of teacher choice and refinement iterations.
Prompts are provided in Appendix~\ref{sec:sft_prompts}.\looseness=-1
\end{minipage}\hfill
\begin{minipage}[t]{0.42\textwidth}
\vspace{-2mm}
\centering
\includegraphics[width=\linewidth]{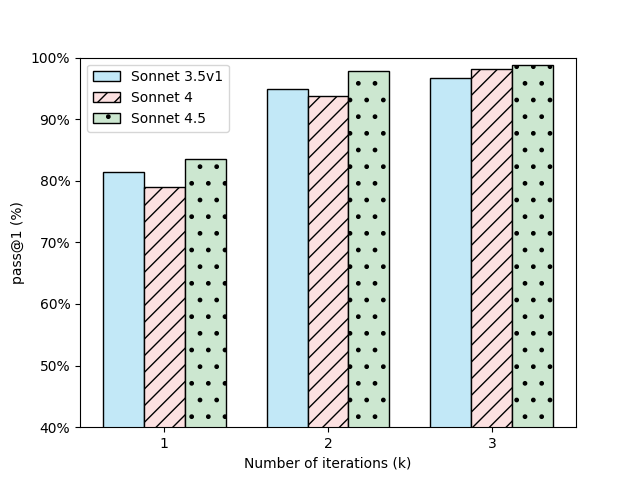}
\caption{Effect of iterative strategy refinement during acceptance sampling: additional refinement iterations increase the probability that the target core solves the problem when conditioned on the proposed strategy.\looseness=-1}
\label{fig:sft-passk}
\end{minipage}
\vspace{-2mm}
\end{figure*}

\paragraph{Results.}
\autoref{tab:gsm8k_results} reports GSM8K accuracy.
First, core-only ICL improves Claude 3.5 Haiku by \textbf{+5.08} points (70.89\%$\rightarrow$75.97\%).
Second, using an unfine-tuned guide yields only a modest gain (70.89\%$\rightarrow$73.46\%), indicating that naive strategies are not consistently executable by the target core.
Third, both SFT and \textsc{Advisor} substantially improve guided performance, with SFT being a strong and sample-efficient baseline.
Finally, our \textsc{ExecTune} recipe achieves the best performance: \textbf{93.56\%} accuracy, improving over \textsc{Advisor} by \textbf{+8.95} points (84.61\%$\rightarrow$93.56\%) and over Claude 3.5 Haiku core-only by \textbf{+22.67} points (70.89\%$\rightarrow$93.56\%).
These gains support our thesis that explicitly optimizing the guide for executability and downstream success is critical in black-box settings.

\begin{table*}[!htbp]
\centering
\caption{Performance on math and code benchmarks. Guided generation with \textsc{ExecTune} achieves the strongest results, demonstrating that optimizing guides for executability improves black-box core performance. FT Recipe denotes the guide fine-tuning recipe.}
\label{tab:main_results}

\begin{subtable}[t]{0.48\textwidth}
\centering
\caption{\textbf{GSM8K (Math)}. \textsc{ExecTune} achieves 93.56\% accuracy; \textbf{+22.67} points over Claude 3.5 Haiku alone.}
\label{tab:gsm8k_results}
\vspace{5pt}
\resizebox{\linewidth}{!}{
\begin{tabular}{@{}cccc@{}}
\toprule
\textbf{Core (API)} & \textbf{Guide} & \textbf{FT Recipe} & \textbf{Acc.} \\
\midrule
\multicolumn{4}{@{}c}{\textit{Baselines (No Guided Generation)}} \\[2pt]
\midrule
Qwen3-1.7B          & —           & —           & 73.39\% \\
Claude 3.5 Haiku    & —           & —           & 70.89\% \\
Claude 3.5 Haiku (ICL) & —        & —           & 75.97\% \\
Claude 3.5 Sonnet   & —           & —           & 90.90\% \\
\midrule
\multicolumn{4}{@{}c}{\textit{Guided Generation (Core = Claude 3.5 Haiku)}} \\[2pt]
\midrule
Claude 3.5 Haiku    & Qwen3-1.7B  & None        & 73.46\% \\
Claude 3.5 Haiku    & Qwen3-1.7B  & SFT         & 81.05\% \\
Claude 3.5 Haiku    & Qwen3-1.7B  & \textsc{Advisor} & 84.61\% \\
Claude 3.5 Haiku    & Qwen3-1.7B  & \textsc{ExecTune} & \textbf{93.56\%} \\
\bottomrule
\end{tabular}
}
\end{subtable}
\hfill
\begin{subtable}[t]{0.48\textwidth}
\centering
\caption{\textbf{KodCode and HumanEval (Code)}. \textsc{ExecTune} enables Claude 3 Haiku to surpass Claude 3.5 Sonnet on KodCode (\textbf{84.00\%} vs 83.33\%) and HumanEval (\textbf{91.46\%} vs 84.15\%).}
\label{tab:kodcode_results}
\vspace{5pt}
\resizebox{\linewidth}{!}{
\begin{tabular}{@{}ccccc@{}}
\toprule
\textbf{Core (API)} & \textbf{Guide} & \textbf{FT Recipe} & \makecell{\textbf{KodCode} \\ \textbf{Pass@1}} & \makecell{\textbf{HumanEval} \\ \textbf{Pass@1}} \\
\midrule
\multicolumn{5}{@{}c}{\textit{Baselines (No Guided Generation)}} \\[2pt]
\midrule
Qwen3-1.7B        & —           & —           & 47.47\% & 71.39\% \\
Claude 3 Haiku    & —           & —           & 68.69\% & 73.17\% \\
Claude 3 Haiku (ICL) & —        & —           & 75.25\% & 68.67\% \\
Claude 3.5 Sonnet & —           & —           & 83.33\% & 84.15\% \\
\midrule
\multicolumn{5}{@{}c}{\textit{Guided Generation (Core = Claude 3 Haiku)}} \\[2pt]
\midrule
Claude 3 Haiku    & Qwen3-1.7B  & None        & 72.56\% & 79.27\% \\
Claude 3 Haiku    & Qwen3-1.7B  & SFT         & 76.73\% & 84.15\% \\
Claude 3 Haiku    & Qwen3-1.7B  & \textsc{Advisor} & 79.29\% & 79.88\% \\
Claude 3 Haiku    & Qwen3-1.7B  & \textsc{ExecTune} & \textbf{84.00\%} & \textbf{91.46\%} \\
\bottomrule
\end{tabular}
}
\end{subtable}

\end{table*}

\subsection{Code Generation (KodCode, HumanEval)}
\label{subsec:code_reasoning}

We train code guides on KodCode-train using the previously introduced recipes (None/SFT/\textsc{Advisor}/\textsc{ExecTune}).
Strategy SFT data construction follows the same acceptance-sampling template as GSM8K, except that success is defined by \emph{unit-test execution}: a proposed strategy is accepted if the target core solution passes the associated tests.
Unless otherwise noted, the black-box core for code is \textbf{Claude 3 Haiku}.

\paragraph{Results on KodCode and HumanEval.}
\autoref{tab:kodcode_results} reports KodCode (in-domain) and HumanEval (out-of-domain) Pass@1.
Core-only ICL substantially improves KodCode (+6.56 points) but \emph{hurts} HumanEval (-4.50 points), consistent with retrieval overfitting to the KodCode training distribution.
Guided generation with an unfine-tuned guide already yields gains over core-only prompting (e.g., +3.87 points on KodCode and +6.10 on HumanEval), but fine-tuning is crucial for robust improvements.
SFT provides a strong baseline and matches the stronger core (Claude 3.5 Sonnet) on HumanEval (84.15\%).
Advisor-style training improves KodCode further but degrades HumanEval relative to SFT, highlighting a generalization trade-off.
Our \textsc{ExecTune} recipe achieves the best overall performance, improving over \textsc{Advisor} by \textbf{+4.71} points on KodCode (79.29\%$\rightarrow$84.00\%) and \textbf{+11.58} points on HumanEval (79.88\%$\rightarrow$91.46\%), and surpasses Claude 3.5 Sonnet on both benchmarks.

\subsection{Qualitative analysis of generated strategies}
\label{sec:qualitative_analysis}

We qualitatively compare strategies from different \GCoP\ guide training recipes on a representative KodCode test problem, keeping the black-box core fixed (\textbf{Claude 3 Haiku}). Base (prompted) and \textsc{Advisor} guides often produce generic or underspecified advice that does not sufficiently constrain the core’s implementation, while SFT guides are typically better structured but can be verbose or present multiple competing directions, leading to inconsistent or incomplete code. In contrast, \textsc{ExecTune} produces the most \emph{executable} strategies that are actionable, interface-aligned, and reliably steer the core to implement the intended solution. Overall, these examples reinforce that gains come from improving \emph{strategy executability} (parseable and faithfully followed plans), not from longer or more elaborate guidance; full prompts and outputs are provided in~\autoref{app:qualitative_analysis}.

\section{Conclusion \& Limitations}
\label{sec:conclusion}
We introduced \GCoP, a unified abstraction for guide-core steering systems under black-box access, and formalized their deployment trade-offs through a cost-sensitive net-utility objective. Our analysis identifies \emph{guide-averaged executability} as the key factor governing the performance gap between composed policies and stronger, more expensive cores. Guided by this perspective, we proposed \textsc{ExecTune}, a two-stage recipe that (i) curates executable strategy supervision via teacher-guided acceptance sampling and (ii) refines guides with structure-aware RL to enforce parseable, reliably executable strategies. Across mathematical reasoning and code generation benchmarks, \textsc{ExecTune} improves accuracy while preserving the cost advantages of cheaper black-box cores, and in several settings approaches or surpasses stronger cores at substantially lower inference cost. \textbf{Limitations:} our experiments focus on single-turn tasks where a single strategy is generated and executed once; we do not evaluate more challenging multi-turn agentic settings that require iterative tool use, long-horizon planning, memory, or recovery from intermediate failures. Extending \GCoP\ and \textsc{ExecTune} to such multi-step environments raises additional questions around strategy updating, credit assignment across turns, and robustness under distribution shift, which we leave to future work.

\newpage

\appendix
\section{Appendix}

\subsection{Method Details}
\label{app:method_details}

This appendix provides the full formalism and details omitted from the condensed~\autoref{sec:method} in the main paper.
We include: (i) the student--teacher mixture model and derivations linking executability to the value gap, (ii) the accepted-strategy target distribution induced by teacher-guided acceptance sampling and its guarantees, and (iii) the full \textsc{ExecTune} post-training procedure, including structure-aware GRPO.

\subsection{Preliminaries: objective and cost}
\label{app:preliminaries}
We model the environment as a finite-horizon controlled process with state space $\mathcal{S}$, action space $\mathcal{A}$, and horizon $H\in\mathbb{N}$.
At each step $t\in\{0,\dots,H-1\}$ the agent observes $s_t\in \mathcal{S}$, selects $a_t\in\mathcal{A}$, receives reward $r_t=r(s_t,a_t)\in[0,R_{\max}]$, and transitions according to the environment dynamics.
An initial state is drawn from $s_0\sim \mu$ (user-query distribution).
For any policy $\pi$, define the $H$-step value
\[
V_H^\pi(s)\triangleq \E\!\left[\sum_{t=0}^{H-1} r(s_t,a_t)\,\middle|\,s_0=s,\pi\right].
\]
Let $T_\pi(s)$ denote the expected total inference-time cost of deploying $\pi$ from $s$ (tokens, latency, or monetary cost).
We optimize the reward--cost tradeoff
\begin{equation}
\label{eq:reward_cost_objective_app}
J_\pi(s)\triangleq V_H^\pi(s)-\lambda T_\pi(s),
\end{equation}
where $\lambda\ge 0$ controls cost sensitivity.

\subsection{\GCoP: Guide--Core policies as a policy family}
\label{app:gcop_family}
Let $\mathcal{Z}$ denote a strategy space (plans, advice, program sketches, structured prompts).
A trainable guide $\pi_g(z\mid s):\mathcal{S}\to \Delta(\mathcal{Z})$ samples a strategy conditioned on state $s$.
A black-box core $\pi_c(a\mid s,z):\mathcal{S}\times\mathcal{Z}\to \Delta(\mathcal{A})$ executes actions conditioned on both $(s,z)$.
Composing guide and core yields the induced action policy
\begin{equation}
\label{eq:composite_policy_app}
\pi_{gc}(a\mid s)\triangleq \sum_{z\in\mathcal{Z}} \pi_g(z\mid s)\,\pi_c(a\mid s,z).
\end{equation}
Throughout, \GCoP\ denotes the \emph{policy family} defined by Eq.~\eqref{eq:composite_policy_app} under the objective Eq.~\eqref{eq:reward_cost_objective_app}, where different instantiations correspond to different guide-training procedures.

\paragraph{When can cheaper \GCoP\ outperform a large model?}
Let $\pi_L$ denote a large and costly baseline policy (e.g., frontier API model).
A sufficient condition for $\pi_{gc}$ to improve net utility over $\pi_L$ from a fixed initial state $s$ is
\begin{equation}
\label{eq:sufficient_condition_state_app}
\lambda T_{\pi_L}(s)-\lambda T_{\pi_{gc}}(s)\;>\;V_H^{\pi_L}(s)-V_H^{\pi_{gc}}(s),
\end{equation}
i.e., the cost advantage must dominate the value gap. The following sections show that the value gap is governed by a single quantity: the guide-averaged executability $\alpha(s)$.

\subsection{Student--teacher mixture model and executability}
\label{app:student_teacher_mixture}

We analyze \GCoP\ as a student system (core) that can match teacher behavior when the guide proposes a \emph{good} (executable) strategy.

\paragraph{Good strategy and success probability.}
Fix a state--strategy pair $(s,z)\in\mathcal{S}\times\mathcal{Z}$.
Let $A\in\mathcal{A}$ be the random action/output produced by the core under $\pi_c(\cdot\mid s,z)$.
Let $G\in\{0,1\}$ indicate whether execution is ``good'' (application-dependent; e.g., passes tests, achieves reward above a threshold).
Define the success probability
\[
q(s,z)\triangleq \Pr(G=1\mid s,z),
\]
and conditional action distributions
$\pi_{\mathrm{good}}(a\mid s,z)\triangleq \Pr(A=a\mid s,z,G=1)$ and
$\pi_{\mathrm{bad}}(a\mid s,z)\triangleq \Pr(A=a\mid s,z,G=0)$.
Then, by total probability,
\begin{equation}
\label{eq:exact_mixture_identity_app}
\pi_c(\cdot\mid s,z)
=
q(s,z)\,\pi_{\mathrm{good}}(\cdot\mid s,z)
+
(1-q(s,z))\,\pi_{\mathrm{bad}}(\cdot\mid s,z).
\end{equation}

\paragraph{Teacher-aligned good execution (assumption).}
We assume that when execution is good, the core matches the teacher policy:
\begin{equation}
\label{eq:teacher_alignment_assumption_app}
\pi_{\mathrm{good}}(\cdot\mid s,z)=\pi_L(\cdot\mid s)\quad \text{for all $(s,z)$}.
\end{equation}
This is a modeling abstraction capturing the idea that ``good'' executions correspond to teacher-level behavior (high reward/validity).

\paragraph{Guide-averaged executability and induced mixture.}
Combining Eq.~\eqref{eq:exact_mixture_identity_app} with the composed policy Eq.~\eqref{eq:composite_policy_app} and the teacher-alignment assumption yields
\begin{equation}
\label{eq:composite_mixture_alpha_app}
\pi_{gc}(\cdot\mid s)
=
\alpha(s)\,\pi_L(\cdot\mid s)
+
(1-\alpha(s))\,\rho_s(\cdot),
\end{equation}
where the mixture weight is the \emph{guide-averaged executability}
\begin{equation}
\label{eq:alpha_def_app}
\alpha(s)\triangleq \E_{z\sim\pi_g(\cdot\mid s)}[q(s,z)]
=
\sum_{z\in\mathcal{Z}} \pi_g(z\mid s)\,q(s,z),
\end{equation}
and $\rho_s$ aggregates bad-execution behavior across strategies:
\begin{equation}
\label{eq:rho_def_app}
\rho_s(\cdot)
\triangleq
\frac{1}{1-\alpha(s)}
\sum_{z\in\mathcal{Z}}
\pi_g(z\mid s)\,(1-q(s,z))\,\pi_{\mathrm{bad}}(\cdot\mid s,z),
\qquad (\alpha(s)<1).
\end{equation}

\paragraph{Value gap bound.}
Let $d_L \triangleq \frac{1}{H}\sum_{t=0}^{H-1} d_{\pi_L,t}$ be the average state visitation distribution under $\pi_L$.
Then:
\begin{theorem}[Value gap of \GCoP]
\label{thm:value_gap_alpha_app}
Under Eq.~\eqref{eq:composite_mixture_alpha_app}, for any $s_0$,
\begin{equation}
\label{eq:value-gap-thm_app}
V_H^{\pi_L}(s_0)-V_H^{\pi_{gc}}(s_0)\;\le\;2HR_{\max}\,\E_{s\sim d_L}[1-\alpha(s)].
\end{equation}
\end{theorem}

\paragraph{Utility condition.}
Combining Eq.~\eqref{eq:value-gap-thm_app} with the net-utility objective implies that \GCoP\ improves net utility whenever the cost advantage dominates the executability shortfall:
\begin{equation}
\label{eq:utility_condition_alpha_app}
\lambda\!\left(T_{\pi_L}(s_0)-T_{\pi_{gc}}(s_0)\right)
\;>\;
2HR_{\max}\,\E_{s\sim d_L}[1-\alpha(s)].
\end{equation}
This motivates guide training procedures that explicitly increase $\alpha(s)$ on task-relevant states.

\subsection{Instantiations of \GCoP}
\label{app:gcop_instantiations}
Different methods correspond to different choices of $\pi_g$:
\begin{itemize}
\item \textbf{\GCoP(Base):} fixed prompted guide emitting strategies in a prescribed format.
\item \textbf{\GCoP(SFT):} guide trained by supervised fine-tuning on strategy data.
\item \textbf{\GCoP(\textsc{Advisor}):} guide trained using the recipe proposed by~\cite{asawa2025train}.
\item \textbf{\GCoP(\textsc{ExecTune}) (ours):} guide initialized by SFT checkpoint on accepted strategies and refined via structure-aware GRPO to explicitly optimize executability.
\end{itemize}

\subsection{Accepted-strategy target distribution via teacher-guided acceptance sampling}
\label{app:accepted_sampling}

\paragraph{Curation pipeline (expanded).}
We operationalize executable strategy supervision using teacher-guided acceptance sampling (Fig.~\ref{fig:sft_recipe}):
a strong teacher proposes candidate strategies for each problem instance, and the target core executes conditioned on each strategy.
An environment/validator assigns success/failure (and potentially feedback). We retain strategies that pass a validation threshold, optionally iterating teacher refinement when failures occur.
The resulting accepted (problem, strategy) pairs form the SFT corpus.

\paragraph{Accepted-strategy distribution.}
Fix state $s$. Let $z\sim \pi_L(\cdot\mid s)$ be a teacher-proposed strategy.
We evaluate it with $K$ i.i.d.\ trials (or stochastic executions) to obtain an empirical success $\hat q_K(s,z)$.
We accept iff $\hat q_K(s,z)\ge \tau$, for threshold $\tau\in(0,1)$.
This induces the accepted distribution $\pi_{\mathrm{acc}}(\cdot\mid s)$:
\[
\pi_{\mathrm{acc}}(z\mid s)\propto \pi_L(z\mid s)\cdot \mathbf{1}\{\hat q_K(s,z)\ge \tau\}.
\]
Let $A_s\triangleq \Pr_{z\sim\pi_L(\cdot\mid s)}(\text{accept})$ denote the acceptance rate.

\paragraph{Guarantee: acceptance increases expected success.}
\begin{theorem}[Accepted strategies have high expected success]
\label{thm:accepted_high_q_app}
Fix $s$ and accept iff $\hat q_K(s,z)\ge\tau$. For any $\eta>0$, define $\delta\triangleq e^{-2K\eta^2}$.
Then
\[
\Pr_{z\sim\pi_{\mathrm{acc}}(\cdot\mid s)}\!\big(q(s,z)\le \tau-\eta\big)\;\le\;\frac{\delta}{A_s},
\]
and consequently
\begin{equation}
\label{eq:accepted_q_lower_bound_app}
\E_{z\sim \pi_{\mathrm{acc}}(\cdot\mid s)}[q(s,z)]
\;\ge\;
(\tau-\eta)\Big(1-\frac{\delta}{A_s}\Big).
\end{equation}
\end{theorem}

\paragraph{Implication for executability.}
Training the guide to imitate $\pi_{\mathrm{acc}}(\cdot\mid s)$ via SFT increases $\alpha(s)=\E_{z\sim\pi_g}[q(s,z)]$ toward the accepted-strategy expectation in Eq.~\eqref{eq:accepted_q_lower_bound_app}, up to imitation error.
This is the first stage of \textsc{ExecTune}.

\subsection{\textsc{ExecTune} details}
\label{app:exectune_details}

\paragraph{Stage 1: SFT on executable strategies.}
We initialize the guide $\pi_g$ by supervised fine-tuning on accepted strategy pairs $(s,z)\sim \pi_{\mathrm{acc}}(\cdot\mid s)$, using a fixed output schema that includes a required \texttt{<strategy>} block.
This biases generation toward strategies that the target core can execute, directly improving $\alpha(s)$.

\subsection{Structure-aware GRPO for guide training}
\label{app:structure_aware_grpo}

\paragraph{RL view.}
We treat the \GCoP\ system as the deployed policy: the guide samples $z\sim\pi_g(\cdot\mid s)$, the core executes conditioned on $(s,z)$, producing an outcome evaluated by the environment/validator to yield episodic reward $R$ (e.g., undiscounted return).
We apply GRPO updates to the guide parameters using this composite reward signal.

\paragraph{Why structure-aware shaping is necessary.}
Task reward alone can yield degenerate behavior: the guide may omit the explicit strategy or hide it in free-form text, while the core still sometimes succeeds.
Such solutions do not improve the reusable guide--core interface and do not reliably increase executability in the sense required by the mixture model.

\paragraph{Structure-aware reward.}
We enforce an explicit interface by requiring a well-formed \texttt{<strategy>}...\texttt{</strategy>} block (optionally alongside \texttt{<thinking>}...\texttt{</thinking>}).
Let
\[
\mathbb{I}_{\mathrm{str}} \triangleq \mathbf{1}\{\texttt{<strategy>} \text{ block present and syntactically valid}\}.
\]
We use the shaped reward
\begin{equation}
\label{eq:structure_aware_reward_app}
\tilde R
\triangleq
R\cdot \mathbb{I}_{\mathrm{str}}
+\beta\,\mathbb{I}_{\mathrm{str}},
\end{equation}
where $\beta\ge 0$ is a small bonus for producing a valid strategy block.
Thus malformed or missing strategies receive $\tilde R=0$ even if the final answer is correct, preventing ``implicit strategy'' collapse.

\paragraph{Connection to executability.}
This shaping explicitly encourages parseable strategies and reduces brittle behavior.
Empirically, it increases the probability that the guide emits usable strategies and improves downstream execution reliability, i.e., increases $\alpha(s)$, tightening the gap bound in Theorem~\ref{thm:value_gap_alpha_app}.

\subsubsection{Judge-based strategy shaping}
\label{app:judge_shaping}

\paragraph{Parsed-strategy judge score.}
Let $z$ denote the strategy parsed from the guide output. We compute a scalar judge score
\[
J(s,z)\in[0,1],
\]
using an LLM-as-a-Judge that evaluates (i) the usefulness/specificity of the strategy for solving $s$, and (ii) whether $z$ improperly contains the final answer (e.g., a full code implementation), which would break the intended guide--core separation.
We treat $J(s,z)$ as a reward shaping signal used only for training the guide.

\paragraph{Non-leakage constraint.}
If the judge flags direct answer leakage, we set $J(s,z)=0$ (or equivalently apply an additional penalty), ensuring that leaked strategies are not reinforced even when the downstream execution succeeds.
This promotes reusable high-level strategies rather than copying final solutions into the strategy channel.

\subsubsection{No-negative-behavior shaping (anti-regression)}
\label{app:anti_regression}

\paragraph{Baseline-vs-guided core comparison.}
To ensure that guidance does not induce harmful behavior, we compare the reward of the core running \emph{without} a strategy to the reward when conditioned on the guide's strategy.
Let $y_0 \sim \pi_c(\cdot \mid s)$ denote the core output without guidance and $y_z \sim \pi_c(\cdot \mid s,z)$ denote the output with guidance.
Let $r(\cdot)$ be the task reward computed by the environment/validator. Define
\begin{equation}
\label{eq:deltaR_app}
\Delta R(s,z)
\triangleq
\E\big[r(y_z)\big]-\E\big[r(y_0)\big],
\end{equation}
where the expectation is over core sampling (and any environment stochasticity).
When $\Delta R(s,z)<0$, the strategy has degraded the core's behavior relative to the unguided baseline.

\paragraph{Penalty form.}
We penalize only degradations via the hinge $[-\Delta R]_+$, yielding a term $-\kappa[-\Delta R(s,z)]_+$ in the shaped reward.
This avoids suppressing beneficial strategies ($\Delta R\ge 0$) while discouraging brittle advice that reduces accuracy.

\subsubsection{Final shaped reward for GRPO}
\label{app:final_reward}

Let $\mathbb{I}_{\mathrm{str}}$ indicate whether the guide output contains a well-formed \texttt{<strategy>} block, and let $R$ denote the episodic task reward obtained from executing the core conditioned on the strategy (i.e., from the guided rollout).
Our full shaped reward is
\begin{equation}
\label{eq:exectune_reward_app}
\tilde R
=
\mathbb{I}_{\mathrm{str}}
\Big(
R + \beta + \gamma J(s,z) - \kappa[-\Delta R(s,z)]_+
\Big),
\end{equation}
with hyperparameters $\beta,\gamma,\kappa\ge 0$.
This shaping enforces (i) explicit structured strategies, (ii) high-quality non-leaky strategies, and (iii) a conservative ``do-no-harm'' constraint relative to the unguided core.

\paragraph{Practical notes (implementation-level).}
In our implementation, $\mathbb{I}_{\mathrm{str}}$ is computed via a deterministic parser that checks:
(i) both opening/closing tags exist, (ii) tags are properly nested, and (iii) the extracted strategy is non-empty and within a token budget.
We optionally add additional interface constraints (e.g., enumerated steps, JSON schema) as additional binary indicators; the main paper uses the minimal \texttt{<strategy>} schema for generality.

\subsection{Putting it together: condensed algorithmic description}
\label{app:algorithm_exectune}

\paragraph{\textsc{ExecTune} (informal).}
(1) Use a strong teacher to propose candidate strategies; execute with the target core and accept those that pass validation to form $\pi_{\mathrm{acc}}$ and an SFT corpus.
(2) SFT the guide on accepted strategies to initialize $\pi_g$.
(3) Run structure-aware GRPO on the guide using shaped reward $\tilde R$ in Eq.~\eqref{eq:structure_aware_reward_app}, where the environment/validator reward is computed from the core execution.
This explicitly aligns guide generation with downstream executability, improving $\alpha(s)$ and thus net utility under Eq.~\eqref{eq:reward_cost_objective_app}.

\subsection{Proofs}
\label{app:theory}

\subsubsection{Proof of \cref{thm:value_gap_alpha_main}}
\begin{proof}
From $\pi_{gc}(\cdot\mid s)=\alpha(s)\pi_L(\cdot\mid s)+(1-\alpha(s))\rho_s(\cdot)$ we have
\[
\pi_{gc}(\cdot\mid s)-\pi_L(\cdot\mid s)=(1-\alpha(s))\big(\rho_s(\cdot)-\pi_L(\cdot\mid s)\big).
\]
Taking total variation and using homogeneity plus $\TV(\rho_s,\pi_L)\le 1$ yields
\begin{equation}
\label{eq:tv_bound_alpha}
\TV\!\big(\pi_{gc}(\cdot\mid s),\pi_L(\cdot\mid s)\big)\;\le\;1-\alpha(s).
\end{equation}
A standard finite-horizon simulation bound for undiscounted returns with rewards in $[0,R_{\max}]$ gives
\begin{equation}
\label{eq:sim_bound}
V_H^{\pi_L}(s_0)-V_H^{\pi_{gc}}(s_0)
\;\le\;
2HR_{\max}\,
\mathbb{E}_{s\sim d_L}\!\left[\TV\!\big(\pi_{gc}(\cdot\mid s),\pi_L(\cdot\mid s)\big)\right].
\end{equation}
Combining \cref{eq:tv_bound_alpha,eq:sim_bound} proves \cref{eq:value-gap-thm_app}.
\end{proof}

\subsubsection{Proof of \cref{thm:accepted_high_q_main}}
\begin{proof}
Let $\hat q_K(s,z)=\frac{1}{K}\sum_{i=1}^K X_i$ be the empirical mean of $K$ i.i.d.\ Bernoulli trials with mean $q(s,z)$, and accept iff $\hat q_K(s,z)\ge\tau$. Define the bad set $\mathcal{B}(s)\triangleq\{z:q(s,z)\le\tau-\eta\}$. If $z\in\mathcal{B}(s)$, then $\tau-q(s,z)\ge\eta$, and Hoeffding's inequality implies
\[
\begin{aligned}
\Pr(\text{accept}\mid s,z)
&=\Pr\!\big(\hat q_K(s,z)\ge\tau\mid s,z\big) \\
&=\Pr\!\big(\hat q_K(s,z)-q(s,z)\ge \tau-q(s,z)\mid s,z\big) \\
&\le \Pr\!\big(\hat q_K(s,z)-q(s,z)\ge \eta\mid s,z\big) \\
&\le e^{-2K\eta^2}
= \delta.
\end{aligned}
\]

Now use the definition of the accepted distribution $\pi_{\mathrm{acc}}(z\mid s)\propto \pi_L(z\mid s)\Pr(\text{accept}\mid s,z)$ with normalization $A_s$. Then
\begin{align}
\Pr_{z\sim\pi_{\mathrm{acc}}(\cdot\mid s)}(z\in\mathcal{B}(s))
&=
\sum_{z\in\mathcal{B}(s)} \pi_{\mathrm{acc}}(z\mid s)
=
\sum_{z\in\mathcal{B}(s)} \frac{\pi_L(z\mid s)\Pr(\text{accept}\mid s,z)}{A_s} \nonumber\\
&\le
\sum_{z\in\mathcal{B}(s)} \frac{\pi_L(z\mid s)\delta}{A_s}
\le
\frac{\delta}{A_s},
\label{eq:bad_mass_proof}
\end{align}
which is $\Pr_{z\sim\pi_{\mathrm{acc}}(\cdot\mid s)}\!\big(q(s,z)\le \tau-\eta\big)\;\le\;\frac{\delta}{A_s}$. For the expectation bound, let $\mathcal{G}(s)=\mathcal{B}(s)^c=\{z:q(s,z)>\tau-\eta\}$. Since $q(s,z)\ge 0$ everywhere and $q(s,z)\ge\tau-\eta$ on $\mathcal{G}(s)$,
\begin{align}
\mathbb{E}_{z\sim\pi_{\mathrm{acc}}(\cdot\mid s)}[q(s,z)]
&\ge
\mathbb{E}\!\big[q(s,z)\mathbf{1}\{z\in\mathcal{G}(s)\}\big]
\ge
(\tau-\eta)\Pr(z\in\mathcal{G}(s)) \nonumber\\
&=
(\tau-\eta)\big(1-\Pr(z\in\mathcal{B}(s))\big)
\ge
(\tau-\eta)\Big(1-\frac{\delta}{A_s}\Big),
\label{eq:accepted_q_proof}
\end{align}
which is \eqref{eq:accepted_q_lower_bound_app}.
\end{proof}

\subsection{Ensuring $\delta < A_s$ in practice}
\label{app:empirical_As}

The lower bound in Theorem~\ref{thm:accepted_high_q_main} depends on the ratio $\delta/A_s$, where
$\delta=e^{-2K\eta^2}$ is the validation error term and $A_s=\Pr_{z\sim\pi_L(\cdot\mid s)}(\text{accept})$ is the (unknown)
acceptance rate. To make the bound nonvacuous in practice, we estimate $A_s$ from teacher proposals and form a high-confidence
lower bound. Concretely, draw $M$ independent strategies $z_1,\dots,z_M\sim\pi_L(\cdot\mid s)$, run the acceptance test, and set
$Y_j \triangleq \mathbf{1}\{\text{$z_j$ is accepted}\}$ so that $\mathbb{E}[Y_j]=A_s$. The empirical acceptance rate is
$\widehat A_s \triangleq \frac{1}{M}\sum_{j=1}^M Y_j$. By Hoeffding, for any $\varepsilon>0$,
$\Pr\!\big(A_s \le \widehat A_s - \varepsilon\big)\le e^{-2M\varepsilon^2}$, equivalently with probability at least
$1-e^{-2M\varepsilon^2}$ we have the lower confidence bound $A_s \ge \widehat A_s-\varepsilon$. We then choose the validation
strength (via $K$ and/or $\eta$) so that $\delta \le \widehat A_s-\varepsilon$, i.e.,
$e^{-2K\eta^2}\le \widehat A_s-\varepsilon$. On the same high-probability event this implies $\delta \le A_s$ and hence
$\delta/A_s \le 1$, ensuring the acceptance-sampling lower bound is nonvacuous. More conservatively, selecting
$\delta \le \tfrac{1}{2}(\widehat A_s-\varepsilon)$ guarantees $\delta/A_s \le 1/2$ under the same event, making the factor
$(1-\delta/A_s)$ quantitatively meaningful.

\section{Implementation Details}
\label{app:implementation}
We implement our framework on top of TRL \citep{vonwerra2022trl}, which provides efficient distributed training and a modular implementation of GRPO. We integrate TRL with vLLM for fast inference and large-scale rollout execution, enabling scalable multi-turn training in our experiments. Prompts used for training and evaluation are listed in \autoref{app:prompt}. The Qwen3 base guide models are available for research use under the Apache 2.0 license.

For RL post training with GRPO, we set the KL regularization factor \(\beta = 0.004\), learning rate to \(10^{-6}\), weight decay to \(0.1\), batch-size to $8$, rollouts to $8$ and max-completion-length to $2048$. During evaluation, we set the temperature to~0.6 and top-$p$ to~0.95 for all experiments that use Qwen3 models.

\section{Qualitative analysis of generated strategies}
\label{app:qualitative_analysis}

We provide the full problem statement and the complete strategy outputs for Base, SFT, \textsc{Advisor}, and \textsc{ExecTune} below.

\begin{tcolorbox}[ enhanced, colback=blue!5, colframe=blue!30, coltitle=white, fonttitle=\bfseries, title=Problem, attach boxed title to top left={yshift=-2mm, xshift=5mm}, boxed title style={colback=blue!30} ] When you are allocated a room for a programming competition, you have access to a bookshelf with M books. You have a list of k books you are particularly interested in reading during your free time. You want to ensure you have these books before anyone else takes them. Given two lists: the first list contains the titles of the M books currently on the bookshelf, and the second list consists of the k book titles you are interested in. Write a function that checks whether all the titles in your list are present on the bookshelf. \end{tcolorbox}

\begin{tcolorbox}[ enhanced, colback=yellow!10, colframe=orange!40, coltitle=black, fonttitle=\bfseries, title=Strategy from Baseline Guide ] Check if every book in the interested\_books list is present in the bookshelf list. If all are present, return True; otherwise, return False. \end{tcolorbox}

\begin{tcolorbox}[ enhanced, colback=purple!8, colframe=purple!50, coltitle=black, fonttitle=\bfseries, title=Strategy from SFT Guide ] \textbf{Guidance for the All Books Available Problem:} \textbf{Step 1: Understand the Problem} \begin{itemize} \item You need to verify if every book in your interest list exists on the bookshelf \item Return True if all interested books are present; False otherwise \end{itemize} \textbf{Step 2: Consider Data Structures} \begin{itemize} \item Think about what data structure would make checking `membership" (is this book present?) efficient \item Consider what data structure would make checking `subset" (are all my books present?) efficient \end{itemize} \textbf{Step 3: Key Insight - Set Theory} \begin{itemize} \item This is fundamentally a subset relationship problem \item If every element in set A exists in set B, then A is a subset of B \item Python's set data structure has built-in methods for subset checking \end{itemize} \textbf{Step 4: Solution Approach} \begin{enumerate} \item Convert both lists to sets (this enables efficient membership testing) \item Use set operations to determine if one set is entirely contained within another \item Return True if all interested books are present, False otherwise \end{enumerate} \textbf{Step 5: Consider Edge Cases} \begin{itemize} \item What if interested\_books has duplicates? (Set conversion handles this naturally) \item What if interested\_books contains books not on the bookshelf? (Set difference will show this) \end{itemize} \textbf{Step 6: Python Set Methods to Consider} \begin{itemize} \item \texttt{.issubset()} or equivalent operators to check if all elements of one set are in another \item Set intersection vs. difference vs. union - think about which relates to `all elements present" \end{itemize} \textbf{Key Hint:} Instead of iterating through each interested book and checking if it exists in the bookshelf, leverage Python's set operations for a more elegant and Pythonic solution. \end{tcolorbox}

\begin{tcolorbox}[ enhanced, colback=orange!10, colframe=orange!60, coltitle=black, fonttitle=\bfseries, title=Strategy from \textsc{Advisor} model ] Check if all books in interested\_books are present in the bookshelf. For each book in interested\_books, if any book is not found in the bookshelf, return False. Otherwise, return True. Use the in' keyword for membership check. The function should handle duplicates in interested\_books by checking each book individually. \end{tcolorbox}

\begin{tcolorbox}[ enhanced, colback=green!8, colframe=green!60!black, coltitle=white, fonttitle=\bfseries, title=Strategy from \textsc{ExecTune} Guide ] \textbf{\large Step-by-step Guidance for Checking Book Availability:} \textbf{1. Understand the Core Task:} \begin{itemize} \item You need to verify that every book in your \texttt{interested\_books} list exists somewhere in the \texttt{bookshelf} list. \item Example: If interested\_books is \texttt{["A", "B"]} and bookshelf is \texttt{["B", "C", "A"]}, return \texttt{True}. \end{itemize} \textbf{2. Key Question to Guide Toward:} \begin{itemize} \item `How can I efficiently check if each book in my interest list exists in the bookshelf collection?" \item Consider: What Python operations or methods can check membership or presence? \end{itemize} \textbf{3. Alternative Approaches to Consider:} \begin{itemize} \item Set operations: Sets have mathematical operations that can help identify common elements. \item Membership testing: Python's \texttt{in} operator can check if an element exists in a collection. \item List methods: There are built-in methods for searching within lists. \end{itemize} \textbf{4. Edge Cases to Think About:} \begin{itemize} \item Empty \texttt{interested\_books} list (should return \texttt{True} by convention) \item Duplicate books in bookshelf \item Case sensitivity (problem assumes case-sensitive matching) \end{itemize} \textbf{5. Pythonic Approach:} Leverage Python's set operations or membership testing for clean, readable code. The gold solution uses set intersection, but list membership testing is also valid. \vspace{2mm} \textit{\textbf{Hint:} Python has built-in operators/methods that can check if all elements of one collection exist in another. Consider \texttt{set} operations or the \texttt{in} operator for membership checking.} \end{tcolorbox}

\section{Prompts for training and evaluation}
\label{app:prompt}

\begin{tcolorbox}[
  enhanced,
  breakable,
  colback=ngreen,
  colframe=green!20!white,
  coltitle=black,
  title=\textbf{\textsc{ExecTune} System Prompt (code generation)},
]
You are an expert Python programmer. Generate guidance for the given coding problem.\\
Your response should be a strategy that helps guide the solution, not the implementation itself. \texttt{OUTPUT FORMAT:} You MUST format your response as follows:\\
1. First, use \texttt{<think></think>} tags for your internal reasoning\\
2. Then, use \texttt{<strategy></strategy>} tags for your final strategy\\\\
Example format: \\ \texttt{<think>}Let me analyze the problem...\texttt{</think><strategy>}strategy to solve the problem\texttt{</strategy>}

\end{tcolorbox}

\begin{tcolorbox}[
  enhanced,
  breakable,
  colback=nblue,
  colframe=blue!20!white,
  coltitle=black,
  title=\textbf{\textsc{ExecTune} Evaluation System Prompt (mathematical reasoning)},
]

You are an expert mathematician. A strategy/hint has been provided to help guide your solution. You should consider this strategy if it seems reasonable and helpful, but you are free to use your own approach if you believe it's better. Your goal is to find the correct answer.
\\
The expected format for the final answer is \texttt{\\boxed\{{answer}\}}

\end{tcolorbox}

\begin{tcolorbox}[
  enhanced,
  breakable,
  colback=nblue,
  colframe=blue!20!white,
  coltitle=black,
  title=\textbf{\textsc{ExecTune} Evaluation System Prompt (code generation)},
]

You are an expert Python programmer. A strategy/hint has been provided to help guide your solution. You should consider this strategy if it seems reasonable and helpful, but you are free to use your own approach if you believe it's better. Your goal is to write correct, working code.\\
\\
Write your full implementation (restate the function signature). The expected format for the final answer is \texttt{```python {python code goes here}```}

\end{tcolorbox}

\section{Prompts for SFT data generation}
\label{sec:sft_prompts}
\begin{tcolorbox}[
  enhanced,
  breakable,
  colback=ngreen,
  colframe=green!20!white,
  coltitle=black,
  title=\textbf{Prompt to generate strategies},
]
You are an expert Python programmer. Provide a high-level strategy, key insights, or solving approach for the given coding problem. Your response should be a brief strategy that helps guide the solution, not the implementation itself.
OUTPUT FORMAT:
You MUST format your response as follows:
\begin{itemize}
    \item First, use \texttt{<think>...</think>} tags for your internal reasoning
    \item Then, use \texttt{<strategy>...</strategy>} tags for your final strategy
  \end{itemize}
Example format:

\texttt{<think>Let me analyze the problem...</think>
<strategy>Your concise strategy here</strategy>}
\newline
\texttt{<problem>}
\{problem\}
\texttt{</problem>}

\texttt{<solution>}
\{solution\}
\texttt{</solution>}
\end{tcolorbox}

\begin{tcolorbox}[
  enhanced,
  breakable,
  colback=ngreen,
  colframe=green!20!white,
  coltitle=black,
  title=\textbf{Prompt to generate solution},
]
You are a software engineering expert who generate coding solutions for a coding problem based on some hints provided to solve the problem.
You task is to read the problem description carefully and leverage the provided hints to generate an optimal coding solution for the problem.
Use the above debugging strategy as additional context to help solve the problem more effectively. Make the solution concise. You need to import all necessary packages.
Your solution should be in

\begin{verbatim}
```python
# YOUR IMPLEMENTATION HERE
```
\end{verbatim}

\texttt{<problem>}
\{user\_problem\}
\texttt{</problem>}

\texttt{<hints>}
\{hints\}
\texttt{</hints>}

\end{tcolorbox}

\begin{thebibliography}{36}
\providecommand{\natexlab}[1]{#1}
\providecommand{\url}[1]{\texttt{#1}}
\expandafter\ifx\csname urlstyle\endcsname\relax
  \providecommand{\doi}[1]{doi: #1}\else
  \providecommand{\doi}{doi: \begingroup \urlstyle{rm}\Url}\fi

\bibitem[Achille \& Soatto(2026)Achille and Soatto]{achille2025ai}
Alessandro Achille and Stefano Soatto.
\newblock {AI} agents as universal task solvers.
\newblock \emph{Entropy (also arXiv:2510.12066)}, 2026.

\bibitem[Agarwal et~al.(2024)Agarwal, Singh, Dani, Magazine, Ganu, and Nambi]{agarwal2024promptwizard}
Eshaan Agarwal, Joykirat Singh, Vivek Dani, Raghav Magazine, Tanuja Ganu, and Akshay Nambi.
\newblock Promptwizard: Task-aware prompt optimization framework.
\newblock \emph{arXiv preprint arXiv:2405.18369}, 2024.

\bibitem[Agrawal et~al.(2025)Agrawal, Tan, Soylu, Ziems, Khare, Opsahl-Ong, Singhvi, Shandilya, Ryan, Jiang, et~al.]{agrawal2025gepa}
Lakshya~A Agrawal, Shangyin Tan, Dilara Soylu, Noah Ziems, Rishi Khare, Krista Opsahl-Ong, Arnav Singhvi, Herumb Shandilya, Michael~J Ryan, Meng Jiang, et~al.
\newblock Gepa: Reflective prompt evolution can outperform reinforcement learning.
\newblock \emph{arXiv preprint arXiv:2507.19457}, 2025.

\bibitem[Asawa et~al.(2025)Asawa, Zhu, Zaharia, Dimakis, and Gonzalez]{asawa2025train}
Parth Asawa, Alan Zhu, Matei Zaharia, Alexandros~G Dimakis, and Joseph~E Gonzalez.
\newblock How to train your advisor: Steering black-box llms with advisor models.
\newblock \emph{arXiv preprint arXiv:2510.02453}, 2025.

\bibitem[Bai et~al.(2022)Bai, Kadavath, Kundu, Askell, Kernion, Jones, Chen, Goldie, Mirhoseini, McKinnon, et~al.]{bai2022constitutional}
Yuntao Bai, Saurav Kadavath, Sandipan Kundu, Amanda Askell, Jackson Kernion, Andy Jones, Anna Chen, Anna Goldie, Azalia Mirhoseini, Cameron McKinnon, et~al.
\newblock Constitutional ai: Harmlessness from ai feedback.
\newblock \emph{arXiv preprint arXiv:2212.08073}, 2022.

\bibitem[Chen et~al.(2023)Chen, Zaharia, and Zou]{chen2023frugalgpt}
Lingjiao Chen, Matei Zaharia, and James Zou.
\newblock Frugalgpt: How to use large language models while reducing cost and improving performance.
\newblock \emph{arXiv preprint arXiv:2305.05176}, 2023.

\bibitem[Chen et~al.(2022)Chen, Ma, Wang, and Cohen]{chen2022program}
Wenhu Chen, Xueguang Ma, Xinyi Wang, and William~W Cohen.
\newblock Program of thoughts prompting: Disentangling computation from reasoning for numerical reasoning tasks.
\newblock \emph{arXiv preprint arXiv:2211.12588}, 2022.

\bibitem[Cheng et~al.(2024)Cheng, Liu, Zheng, Ke, Wang, Dong, Tang, and Huang]{cheng2024black}
Jiale Cheng, Xiao Liu, Kehan Zheng, Pei Ke, Hongning Wang, Yuxiao Dong, Jie Tang, and Minlie Huang.
\newblock Black-box prompt optimization: Aligning large language models without model training.
\newblock In \emph{Proceedings of the 62nd Annual Meeting of the Association for Computational Linguistics (Volume 1: Long Papers)}, pp.\  3201--3219, 2024.

\bibitem[Cobbe et~al.(2021)Cobbe, Kosaraju, Bavarian, Chen, Jun, Kaiser, Plappert, Tworek, Hilton, Nakano, Hesse, and Schulman]{gsm8k}
Karl Cobbe, Vineet Kosaraju, Mohammad Bavarian, Mark Chen, Heewoo Jun, Lukasz Kaiser, Matthias Plappert, Jerry Tworek, Jacob Hilton, Reiichiro Nakano, Christopher Hesse, and John Schulman.
\newblock Training verifiers to solve math word problems.
\newblock \emph{arXiv preprint arXiv:2110.14168}, 2021.

\bibitem[Do et~al.(2025)Do, Doddipatla, and Knill]{do2025effectiveness}
Cong~Thanh Do, Rama~Sanand Doddipatla, and Kate Knill.
\newblock Effectiveness of chain-of-thought in distilling reasoning capability from large language models.
\newblock In \emph{Proceedings of the 18th International Natural Language Generation Conference}, pp.\  833--845, 2025.

\bibitem[Ekbote et~al.(2025)Ekbote, Lingam, Tehrani, Huan, sujay sanghavi, Deoras, and Soatto]{murphy}
Chanakya Ekbote, Vijay Lingam, Behrooz~Omidvar Tehrani, Jun Huan, sujay sanghavi, Anoop Deoras, and Stefano Soatto.
\newblock Murphy: Reflective multi-turn reinforcement learning for self-correcting code generation in large language.
\newblock In \emph{First Workshop on Foundations of Reasoning in Language Models}, 2025.
\newblock URL \url{https://openreview.net/forum?id=x0Ir7cWEiA}.

\bibitem[{\em et al.}(2021{\natexlab{a}})]{liu-etal-2021-dexperts}
Alisa~Liu {\em et al.}
\newblock {DE}xperts: Decoding-time controlled text generation with experts and anti-experts.
\newblock In \emph{Proceedings of the 59th Annual Meeting of the Association for Computational Linguistics and the 11th International Joint Conference on Natural Language Processing (Volume 1: Long Papers)}, pp.\  6691--6706, Online, August 2021{\natexlab{a}}. Association for Computational Linguistics.
\newblock \doi{10.18653/v1/2021.acl-long.522}.
\newblock URL \url{https://aclanthology.org/2021.acl-long.522/}.

\bibitem[{\em et al.}(2025)]{qwen3}
An~Yang {\em et al.}
\newblock Qwen3 technical report, 2025.
\newblock URL \url{https://arxiv.org/abs/2505.09388}.

\bibitem[{\em et al.}(2021{\natexlab{b}})]{humaneval}
Mark~Chen {\em et al.}
\newblock Evaluating large language models trained on code, 2021{\natexlab{b}}.
\newblock URL \url{https://arxiv.org/abs/2107.03374}.

\bibitem[Fernando et~al.(2023)Fernando, Banarse, Michalewski, Osindero, and Rockt{\"a}schel]{fernando2023promptbreeder}
Chrisantha Fernando, Dylan Banarse, Henryk Michalewski, Simon Osindero, and Tim Rockt{\"a}schel.
\newblock Promptbreeder: Self-referential self-improvement via prompt evolution.
\newblock \emph{arXiv preprint arXiv:2309.16797}, 2023.

\bibitem[Gu et~al.(2023)Gu, Dong, Wei, and Huang]{gu2023minillm}
Yuxian Gu, Li~Dong, Furu Wei, and Minlie Huang.
\newblock Minillm: Knowledge distillation of large language models.
\newblock \emph{arXiv preprint arXiv:2306.08543}, 2023.

\bibitem[Hsieh et~al.(2023)Hsieh, Li, Yeh, Nakhost, Fujii, Ratner, Krishna, Lee, and Pfister]{hsieh2023distilling}
Cheng-Yu Hsieh, Chun-Liang Li, Chih-Kuan Yeh, Hootan Nakhost, Yasuhisa Fujii, Alex Ratner, Ranjay Krishna, Chen-Yu Lee, and Tomas Pfister.
\newblock Distilling step-by-step! outperforming larger language models with less training data and smaller model sizes.
\newblock In \emph{Findings of the Association for Computational Linguistics: ACL 2023}, pp.\  8003--8017, 2023.

\bibitem[Jahan \& Sun(2025)Jahan and Sun]{jahan2025black}
Sohely Jahan and Ruimin Sun.
\newblock Black-box behavioral distillation breaks safety alignment in medical llms.
\newblock \emph{arXiv preprint arXiv:2512.09403}, 2025.

\bibitem[Kleinman et~al.(2025)Kleinman, Trager, Achille, Xia, and Soatto]{kleinman2025e1}
Michael Kleinman, Matthew Trager, Alessandro Achille, Wei Xia, and Stefano Soatto.
\newblock {E1}: Learning adaptive control of reasoning effort.
\newblock \emph{NeurIPS Workshop on Efficient Reasoning (also arXiv:2510.27042)}, 2025.

\bibitem[Kolawole et~al.(2024)Kolawole, Dennis, Talwalkar, and Smith]{kolawole2024agreement}
Steven Kolawole, Don Dennis, Ameet Talwalkar, and Virginia Smith.
\newblock Agreement-based cascading for efficient inference.
\newblock \emph{arXiv preprint arXiv:2407.02348}, 2024.

\bibitem[Li et~al.(2025)Li, Zhuang, Qiang, Sun, Dai, Zhang, and Dai]{li2025matryoshka}
ChangHao Li, Yuchen Zhuang, Rushi Qiang, Haotian Sun, Hanjun Dai, Chao Zhang, and Bo~Dai.
\newblock Matryoshka pilot: Learning to drive black-box llms with llms.
\newblock In \emph{The Thirty-ninth Annual Conference on Neural Information Processing Systems}, 2025.

\bibitem[Li et~al.(2024)Li, Gu, Dong, Wang, Cheng, and Wei]{li2024direct}
Yixing Li, Yuxian Gu, Li~Dong, Dequan Wang, Yu~Cheng, and Furu Wei.
\newblock Direct preference knowledge distillation for large language models.
\newblock \emph{arXiv preprint arXiv:2406.19774}, 2024.

\bibitem[Li et~al.(2023)Li, Peng, He, Galley, Gao, and Yan]{li2023guiding}
Zekun Li, Baolin Peng, Pengcheng He, Michel Galley, Jianfeng Gao, and Xifeng Yan.
\newblock Guiding large language models via directional stimulus prompting.
\newblock In \emph{Proceedings of the 37th International Conference on Neural Information Processing Systems}, pp.\  62630--62656, 2023.

\bibitem[Lingam et~al.(2025)Lingam, Tehrani, Sanghavi, Gupta, Ghosh, Liu, Huan, and Deoras]{dot}
Vijay Lingam, Behrooz~Omidvar Tehrani, Sujay Sanghavi, Gaurav Gupta, Sayan Ghosh, Linbo Liu, Jun Huan, and Anoop Deoras.
\newblock Enhancing language model agents using diversity of thoughts.
\newblock In \emph{The 13th International Conference on Learning Representations}, 2025.
\newblock URL \url{https://openreview.net/forum?id=ZsP3YbYeE9}.

\bibitem[Madaan et~al.(2023)Madaan, Tandon, Gupta, Hallinan, Gao, Wiegreffe, Alon, Dziri, Prabhumoye, Yang, et~al.]{madaan2023self}
Aman Madaan, Niket Tandon, Prakhar Gupta, Skyler Hallinan, Luyu Gao, Sarah Wiegreffe, Uri Alon, Nouha Dziri, Shrimai Prabhumoye, Yiming Yang, et~al.
\newblock Self-refine: Iterative refinement with self-feedback.
\newblock \emph{Advances in Neural Information Processing Systems}, 36:\penalty0 46534--46594, 2023.

\bibitem[Ouyang et~al.(2022)Ouyang, Wu, Jiang, Almeida, Wainwright, Mishkin, Zhang, Agarwal, Slama, Ray, et~al.]{ouyang2022training}
Long Ouyang, Jeffrey Wu, Xu~Jiang, Diogo Almeida, Carroll Wainwright, Pamela Mishkin, Chong Zhang, Sandhini Agarwal, Katarina Slama, Alex Ray, et~al.
\newblock Training language models to follow instructions with human feedback.
\newblock \emph{Advances in neural information processing systems}, 35:\penalty0 27730--27744, 2022.

\bibitem[Schick et~al.(2023)Schick, Dwivedi-Yu, Dess{\`\i}, Raileanu, Lomeli, Hambro, Zettlemoyer, Cancedda, and Scialom]{schick2023toolformer}
Timo Schick, Jane Dwivedi-Yu, Roberto Dess{\`\i}, Roberta Raileanu, Maria Lomeli, Eric Hambro, Luke Zettlemoyer, Nicola Cancedda, and Thomas Scialom.
\newblock Toolformer: Language models can teach themselves to use tools.
\newblock \emph{Advances in Neural Information Processing Systems}, 36:\penalty0 68539--68551, 2023.

\bibitem[Shao et~al.(2024)Shao, Wang, Zhu, Xu, Song, Zhang, Li, Wu, and Guo]{grpo}
Zhihong Shao, Peiyi Wang, Qihao Zhu, Runxin Xu, Junxiao Song, Mingchuan Zhang, Y.K. Li, Y.~Wu, and Daya Guo.
\newblock Deepseekmath: Pushing the limits of mathematical reasoning in open language models, 2024.
\newblock URL \url{https://arxiv.org/abs/2402.03300}.

\bibitem[Shin et~al.(2020)Shin, Razeghi, Logan~IV, Wallace, and Singh]{shin2020autoprompt}
Taylor Shin, Yasaman Razeghi, Robert~L. Logan~IV, Eric Wallace, and Sameer Singh.
\newblock {A}uto{P}rompt: {E}liciting {K}nowledge from {L}anguage {M}odels with {A}utomatically {G}enerated {P}rompts.
\newblock In \emph{Proceedings of the 2020 Conference on Empirical Methods in Natural Language Processing (EMNLP)}, pp.\  4222--4235, Online, November 2020. Association for Computational Linguistics.
\newblock \doi{10.18653/v1/2020.emnlp-main.346}.
\newblock URL \url{https://aclanthology.org/2020.emnlp-main.346/}.

\bibitem[von Werra et~al.(2020)von Werra, Belkada, Tunstall, Beeching, Thrush, Lambert, Huang, Rasul, and Gallou{\'e}dec]{vonwerra2022trl}
Leandro von Werra, Younes Belkada, Lewis Tunstall, Edward Beeching, Tristan Thrush, Nathan Lambert, Shengyi Huang, Kashif Rasul, and Quentin Gallou{\'e}dec.
\newblock {TRL: Transformer Reinforcement Learning}.
\newblock \url{https://github.com/huggingface/trl}, 2020.

\bibitem[Xu et~al.(2025)Xu, Liu, Yin, Zhou, and Poovendran]{kodcode}
Zhangchen Xu, Yang Liu, Yueqin Yin, Mingyuan Zhou, and Radha Poovendran.
\newblock {K}od{C}ode: A diverse, challenging, and verifiable synthetic dataset for coding.
\newblock In \emph{Findings of the Association for Computational Linguistics: ACL 2025}, pp.\  6980--7008, Vienna, Austria, July 2025. Association for Computational Linguistics.
\newblock \doi{10.18653/v1/2025.findings-acl.365}.
\newblock URL \url{https://aclanthology.org/2025.findings-acl.365/}.

\bibitem[Yang \& Klein(2021)Yang and Klein]{yang-klein-2021-fudge}
Kevin Yang and Dan Klein.
\newblock {FUDGE}: Controlled text generation with future discriminators.
\newblock In \emph{Proceedings of the 2021 Conference of the North American Chapter of the Association for Computational Linguistics: Human Language Technologies}, pp.\  3511--3535, Online, June 2021. Association for Computational Linguistics.
\newblock \doi{10.18653/v1/2021.naacl-main.276}.
\newblock URL \url{https://aclanthology.org/2021.naacl-main.276/}.

\bibitem[Yao et~al.(2022)Yao, Zhao, Yu, Du, Shafran, Narasimhan, and Cao]{yao2022react}
Shunyu Yao, Jeffrey Zhao, Dian Yu, Nan Du, Izhak Shafran, Karthik~R Narasimhan, and Yuan Cao.
\newblock React: Synergizing reasoning and acting in language models.
\newblock In \emph{The eleventh international conference on learning representations}, 2022.

\bibitem[Zabounidis et~al.(2025)Zabounidis, Golatkar, Kleinman, Achille, Xia, and Soatto]{zabounidis2025re}
Renos Zabounidis, Aditya Golatkar, Michael Kleinman, Alessandro Achille, Wei Xia, and Stefano Soatto.
\newblock Re-forc: Adaptive reward prediction for efficient chain-of-thought reasoning.
\newblock \emph{NeurIPS Workshop on Efficient Reasoning (also arXiv:2511.02130)}, 2025.

\bibitem[Zhang et~al.(2022)Zhang, Zhang, Li, and Smola]{zhang2022automatic}
Zhuosheng Zhang, Aston Zhang, Mu~Li, and Alex Smola.
\newblock Automatic chain of thought prompting in large language models.
\newblock \emph{arXiv preprint arXiv:2210.03493}, 2022.

\bibitem[Zhou et~al.(2022)Zhou, Sch{\"a}rli, Hou, Wei, Scales, Wang, Schuurmans, Cui, Bousquet, Le, et~al.]{zhou2022least}
Denny Zhou, Nathanael Sch{\"a}rli, Le~Hou, Jason Wei, Nathan Scales, Xuezhi Wang, Dale Schuurmans, Claire Cui, Olivier Bousquet, Quoc Le, et~al.
\newblock Least-to-most prompting enables complex reasoning in large language models.
\newblock \emph{arXiv preprint arXiv:2205.10625}, 2022.

\end{thebibliography}
\end{document}